\documentclass[conference, onecolumn, 11pt]{IEEEtran}

\usepackage[margin=1in]{geometry}
\usepackage{cite}
\usepackage{amsmath,amssymb,amsfonts, amsthm}
\usepackage{algorithmic}
\usepackage{graphicx}
\usepackage{textcomp}
\usepackage{xcolor}
\usepackage{adjustbox}

\usepackage[utf8]{inputenc} 
\usepackage[T1]{fontenc}    
\usepackage{hyperref}       
\usepackage{url}            
\usepackage{booktabs}       
\usepackage{nicefrac}       
\usepackage{microtype}      

\usepackage{array,enumitem}
\usepackage{makecell,multirow}
\usepackage{comment}
\usepackage[linesnumbered,ruled,vlined]{algorithm2e}

\SetKwInput{KwInput}{Input}                
\SetKwInput{KwOutput}{Output}

\newcolumntype{L}{>$l<$}

\newtheorem{defn}{Definition}

\newtheorem{prop}{Proposition}
\newtheorem{cor}{Corollary}

\def\BibTeX{{\rm B\kern-.05em{\sc i\kern-.025em b}\kern-.08em
    T\kern-.1667em\lower.7ex\hbox{E}\kern-.125emX}}
\begin{document}

\title{Non-Comparative Fairness for Human-Auditing and Its Relation to Traditional Fairness Notions}

\author{\IEEEauthorblockN{Mukund Telukunta}
\IEEEauthorblockA{\textit{Computer Science Department} \\
\textit{Missouri University of Science and Technology}\\
Rolla, Missouri \\
mt3qb@umsystem.edu}
\and
\IEEEauthorblockN{Venkata Sriram Siddhardh Nadendla}
\IEEEauthorblockA{\textit{Computer Science Department} \\
\textit{Missouri University of Science and Technology}\\
Rolla, Missouri \\
nadendla@umsystem.edu}
}

\maketitle

\begin{abstract}
Bias evaluation in machine-learning based services (MLS) based on traditional algorithmic fairness notions that rely on comparative principles is practically difficult, making it necessary to rely on human auditor feedback. However, in spite of taking rigorous training on various comparative fairness notions, human auditors are known to disagree on various aspects of fairness notions in practice, making it difficult to collect reliable feedback.
This paper offers a paradigm shift to the domain of algorithmic fairness via proposing a new fairness notion based on the principle of non-comparative justice. In contrary to traditional fairness notions where the outcomes of two individuals/groups are compared, our proposed notion compares the MLS' outcome with a desired outcome for each input. This desired outcome naturally describes a human auditor's expectation, and can be easily used to evaluate MLS on crowd-auditing platforms. We show that any MLS can be deemed fair from the perspective of comparative fairness (be it in terms of individual fairness, statistical parity, equal opportunity or calibration) if it is non-comparatively fair with respect to a fair auditor. We also show that the converse holds true in the context of individual fairness. Given that such an evaluation relies on the trustworthiness of the auditor, we also present an approach to identify fair and reliable auditors by estimating their biases with respect to a given set of sensitive attributes, as well as quantify the uncertainty in the estimation of biases within a given MLS. Furthermore, all of the above results are also validated on COMPAS, German credit and Adult Census Income datasets.
\end{abstract}

\newpage
\section{Introduction}\label{sec: Introduction}
In recent years, the rapid advancements in the fields of artificial intelligence (AI) and machine learning (ML) have resulted in the proliferation of algorithmic decision making in many practical applications. Examples include decision-support systems for judges whether or not to release a prisoner on parole \cite{baryjester2015}, automated financial decisions in banks regarding granting or denying loans \cite{koren2016}, and product recommendations by e-commerce websites \cite{smith2017two}. Although these algorithms have had a significant improvement in overall system efficiency, they are also found to be biased and unfair in some sensitive regards which affected people's lives substantially. For example, a study by ProPublica in \cite{angwin2016} showed how recidivism scores computed by COMPAS algorithm were biased in terms of both race and gender. Likewise, racial discrimination was found in many practical algorithmic decision support services such as e-commerce services in online markets \cite{Harvard2016}, and in life insurance premiums \cite{waxmen2018}. 

The biggest challenge in tackling discrimination is that it can occur due to multiple contributing factors, which cannot be measured using a single fairness notion. Broadly, fairness in machine learning based algorithms can be measured using two fundamentally different philosophical notions: \emph{individual fairness} and \emph{group fairness}. Individual fairness notion is developed based on the principle that similar individuals should be treated similarly \cite{dwork2012fairness}. Applications of this fairness notion include designing hiring decisions in job markets based on applicant's experience and skill-set, college admission decisions based on student grades, and loan applications based on FICO scores. However, individual fairness typically relies on context-dependent distance metrics to evaluate similarity between individuals and their respective outcomes, which are usually unknown and difficult to quantify, especially when individuals under consideration belong to different groups/communities. 
On the other hand, group fairness notions address this limitation by evaluating how outcomes are distributed across groups based on a group-conditional metric. Different metrics lead to distinct statistical fairness notions such as statistical parity, equal opportunity and calibration \cite{MoritzOpportunities, zafar2017, ritov2017conditional}. However, statistical fairness notions have been recently found to be incompetent to offer fair services to individuals/subgroups across different groups. In 2016, ProPublica published a seminal article which demonstrated how COMPAS risk tool presents \emph{unfair} recidivism scores to judges \cite{angwin2016} as it predicts a higher false positive rate on the black defendants compared to their white counterparts (i.e. notion of Equalized odds \cite{hardt2016equality}). Although Northpointe (the company that designed COMPAS tool) designed the algorithm to ensure equal positive predicted value on both white and black defendants (i.e. calibration \cite{kleinberg2016inherent}), ProPublica has led to the recognition of a fundamental inconsistency in analyzing algorithmic fairness. This discrepancy were later proved theoretically by two different research groups independently \cite{chouldechova2017fair, kleinberg2016inherent} by showing that it is impossible to satisfy multiple group-fairness notions at the same time. Their study says that the reason behind such discrepancies is due to the frequency with which the blacks and whites were charged with new crimes. If the system has two populations that have unequal base rates, then you can't satisfy both the definitions of fairness (calibration and false positives) at the same time.

The inability to measure discrimination using available fairness notions necessitates our reliance on human auditors. However, human auditors may not always ensure guarantees from the perspective of a specific fairness notion (as needed within the application domain). In an attempt to address this challenge, there is a need to develop novel fairness approaches based on interaction between algorithmic fairness systems and human auditors. However, there are many challenges in designing an interactive framework with both human auditors and algorithmic fairness systems, some of which are listed below:
\begin{enumerate}[label = (\roman*)]

\item \emph{Heterogeneous Fairness Philosophies}: Since data used for machine learning purposes already contain biases in it, algorithms tend to be unfair to certain groups in society. Moreover, algorithms are trained based on a specific fairness metric and there's no one right metric. Similarly, humans also exhibit various biases due to prior experiences in the society and may perceive fairness in a completely different manner. Hence, dealing with different fairness philosophies is one of the major challenges. As a solution, Kearns \emph{et al.} design algorithms for learning classifiers that are fair with respect to an auditor, based on a formulation as a two-player zero-sum game between the Learner and an Auditor. They ask for subgroup fairness using statistical notions over a large number of protected groups where, the Learner strategy corresponds to classifiers that minimize the sum of prediction error and the Auditor is responsible for rectifying the fairness violation of the Learner \cite{kearns2017preventing}. Similarly, Zhang and Neil \cite{zhang2016identifying} also designed audit algorithms for classifiers to identify subgroups where estimated probability of outcome differ significantly from observed probabilities.


\item \emph{Metric Learning}:
The main challenge in the notion of individual fairness is constructing the task-based similarity metric which itself is a non-trivial task in fairness. Hence, researchers have employed trusted auditors to measure the similarity between individuals using an unknown metric \cite{kim2018fairness}. Similarly, Gillen \emph{et al.} in \cite{gillen2018online} assumes the existence of an auditor who is capable of identifying the fairness violations made in an online setting based on an unknown metric. Rothblum and Yona in \cite{yona2018probably} show that an approximate of fairness metric generalizes to new data drawn from an underlying population distribution. Jung \emph{et al.} in \cite{jung2019eliciting} study an offline learning problem with subjective individual fairness which is benefited by human experts. The proposed algorithm obtains feedback from human experts by asking them questions of the form: “should this pair of individuals be treated similarly or not?”. 

\item \emph{Novel Fairness Notions}: Raji \emph{et al.} in \cite{raji2020closing} suggests that auditing the algorithms can be achieved through internal organization development cycle which could help to tackle the ethical issues raised in a company. Their framework is developed by a small team of auditors in a large company who present data and model documentation along with metrics to facilitate auditing in specific contexts.
\end{enumerate}

This paper mainly focuses on the identification of effective human auditors in terms of their biases with respect to sensitive attributes. Unfortunately, both individual fairness and group fairness notions are insufficient as they are based only on \emph{comparative justice} principles, i.e. they compare individuals/groups based on the similarity of their treatment/outcomes. In order to effectively identify fair auditors and quantify their biases, we rely on another distinct fairness notion called \emph{non-comparative justice} \cite{levine2005comparative, feinberg1974noncomparative, montague1980comparative}, which states that every individual is treated precisely based on their own personal attributes and merits regardless of how other individuals are treated/affected by the same service. In other words, we assume that each auditor assesses the input attributes, constructs a desired outcome, and compares the algorithm's outcome with their own assessment. For example, in order to evaluate the algorithm behind college admissions, an auditor might consider grades/marks, projects, and internships in an individual's profile to construct their desired admission decision and compares it with the algorithm's outcome. 
If both the system's and auditor's outcomes are similar (dissimilar), then the system's classifier will be deemed fair (unfair) by the auditor. 

The remainder of the paper is organized as follows. Firstly, in Section \ref{sec: NC Fairness}, we show that fairness evaluations based on non-comparative justice principles is fundamental to achieving both individual and group fairness (esp. statistical parity) notions. We prove that a system/entity satisfies comparative justice if it satisfies non-comparative fairness with respect to the auditor. We also show that converse holds true in the case of individual fairness. However, auditor evaluations need to be taken with a grain of salt because the success of this framework depends on the assumption that the auditor is intrinsically unbiased. Therefore, in Section \ref{sec: Unknown Biases}, we also present a framework to identify fair auditors along with quantifying their biases towards sensitive attributes. Finally, in Section \ref{sec: Validation}, we also validate our findings on three real datasets, namely COMPAS, Adult Income Census and German credit datasets.

\begin{figure}
    \centering
    \includegraphics[width = \textwidth]{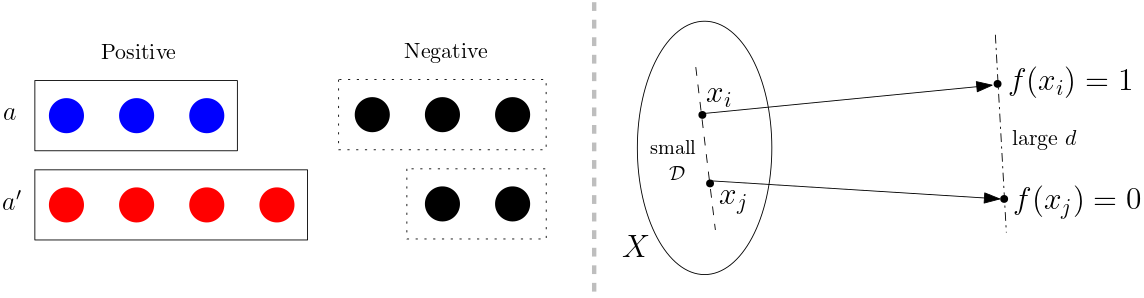}
    \caption{Comparative Notions of Fairness. \textbf{Left:} Unequal number of positive outcomes given for groups $a$ and $a'$, thereby conflicting statistical parity. \textbf{Right:} Two similar individuals (small $\mathcal{D}$) are given different outputs (large $d$) resulting in violation of individual fairness.}
    \label{fig: GFvsIF}
\end{figure}



\section{Comparative Justice for Algorithmic Fairness}\label{sec: Comparative Fairness}

Comparative justice arguments rely on either a comparison or a contrast between the way in which some system has treated two or more individuals or groups. The principle of comparative justice can be formulated as a combination of two axioms: (1) similar cases must be treated similarly, and (2) dissimilar cases must be treated differently. In some contexts, the system might treat individuals differently even though they should be treated similarly. This leads to \emph{comparative injustice}. For instance, imagine a situation where two individuals approach a corporate bank to apply for a loan. Even though both of them have relatively similar qualifications such as credit history, employment, and salary, the bank decides to grant the loan for only one of the individuals. This argument violates the principle of comparative justice because similar individuals are not treated similarly. The bank should either grant both or refuse both. On the contrary, suppose there exist two industries $A$ and $B$ which cause pollution to the environment in varying amounts. Though $B$ produces twice the amount of pollution compared to $A$, the government imposes the same amount of tax on both the industries. Since dissimilar cases are not treated differently, the argument violates the principle of comparative justice.  

As mentioned earlier, algorithmic fairness literature has focused on the principle of comparative justice, mainly, group fairness and individual fairness. The notion of group fairness seek for parity of some statistical measure across all the protected attributes present in the data. Different versions of group-conditional metrics led to different group definitions of fairness. For example, statistical parity \cite{dwork2012fairness} can be formally defined as follows.
\begin{defn}[Statistical Parity]
Given protected attributes $A$, if $f$ is a predictor and $Y = f(X)$, where, $X$ is the multi-attribute variable and $Y$ is the outcome, $f$ achieves statistical parity when
\begin{equation}
\mathbb{P}[f(x) = y \ | \ A = a] = \mathbb{P}[f(x) = y \ | \ A = a']
\end{equation}
holds true for all $y \in Y$, $x \in X$, and $a, a' \in A$.
\label{Defn: GF}
\end{defn}

Similarly, the notion of equal opportunity \cite{MoritzOpportunities} states that the true positive rate should be the same for all the groups which can be defined as follows.
\begin{defn}[Equal Opportunity]
We say that a binary predictor $f$ satisfies equal opportunity with respect to set of protected attributes $A$ and outcome, $Y = f(x)$, if
\begin{equation}
\mathbb{P}[f(x) = 1 \ | \ y = 1, A = a] = \mathbb{P}[f(x) = 1 \ | \ y = 1, A = a'].
\end{equation}
\label{Defn: EO}
\end{defn}
Another example for group-conditional metric is calibration \cite{kleinberg2016inherent, chouldechova2017fair} which ensures positive predictive value across different groups.
\begin{defn}[Calibration]
A binary predictor $f$ satisfies calibration given a set of protected attributes $A$ and outcome, $Y = f(x)$, if
\begin{equation}
\mathbb{P}[y = 1 \ | \ f(x) = 1, A = a] = \mathbb{P}[y = 1 \ | \ f(x) = 1, A = a'].
\end{equation}
\label{Defn: Calibration}
\end{defn}


On the contrary, the notion of individual fairness \cite{dwork2012fairness} states that similar individuals should be treated similarly concerning a particular task. The notion can be defined as follows.
\begin{defn}[Individual Fairness]
Given any two individuals $x_i, x_j \in \mathcal{X}$ with $\mathcal{D} \left( x_i,  x_j \right) \leq \kappa$, then the classifier $f$ is $(\kappa, \delta)$-individually fair if $d \Big( f(x_i), f(x_j) \Big) \leq \delta$. 
\label{Defn: IF}
\end{defn}



Evidently, the notion of individual fairness compares two individuals. However, the notion does not enforce on treating dissimilar individuals. Also, the formulation relies on a suitable similarity metric which is difficult to construct in reality. As pointed out in \cite{speicher2018unified}, individual fairness does not regard individual's merits or desires for different decisions. Hence, \cite{speicher2018unified} proposed a way to measure individual unfairness using inequality indices and taking a person's merit into consideration. On the other hand, \cite{kusner2017} developed a novel framework from causal inference, where a decision or an outcome is considered fair for an individual, only if it is the same in the actual as well as the counterfactual world. For more details, interested readers may refer to a detailed survey on fairness in algorithmic decision-making in \cite{Lepri2018, chouldechova2018frontiers}.

\section{Non-comparative Fairness in Human Judgements}\label{sec: NC Fairness}
Comparative fairness notions are defined based on a comparison between two persons, or two groups of people. However, non-comparative justice notions deviate from this assumption and define justice based on the appropriate treatment to each individual.

The principle of non-comparative justice can be formulated as, "treat each person as he/she deserves or merits". Such an argument does not depend upon any comparison or contrast with the way in which some system treats two or more individuals or groups. For instance, historically, the African-Americans are said to commit more crimes when compared to other races \cite{colorofcrime}. As a result, the recidivism prediction tools are biased towards African-Americans even though the severity of their crimes is less compared to other races. Naturally, tools which are trained on such historical information reflect similar biases. Hence, as a solution, we demonstrate the notion of non-comparative fairness with the help of an expert auditor who classifies inputs based on fair judgements regardless of the historical information or any comparisons. We assume that the auditor employs an intrinsic fair relation to classify the individuals. The proposed notion can be utilized to identify whether an unknown system is fair by comparing system's outcomes with auditor judgements. We define non-comparative fairness formally as follows.

\begin{defn}[$\epsilon$-Noncomparative Fairness]
Let $f$ denote a fair assessment (i.e. input-output relationship) of a given system, i.e. $y = f(x)$, which is evaluated subjectively by an expert auditor. If $g$ is an alternative representation, i.e. $\tilde{y} = g(x)$ (e.g. classification algorithms minimizing loss function, recommender systems), then $g$ is called $\epsilon-$noncomparatively fair w.r.t. $f$ if \begin{equation}
\displaystyle d \left( g(x), f(x) \right) < \epsilon, \text{ for all } x \in \mathcal{X}.
\end{equation}
\label{Defn: Noncomparative Fairness}
\end{defn}
Though the above notion compares the auditor's fair judgements $f$ with the classifier/recommender $g$, it does not compare two different individuals thereby adhering to the principle of non-comparative justice. Figure \ref{fig:NCF Architecture} represents the procedure of identifying fair/unfair system with respect to the auditor. The auditor reveals a binary fair evaluation $s \in \{0, 1\}$ indicating whether the classifier is fair or not. In other words, if the auditor's judgements are similar to classifier outcome then the classifier is said to be fair. However, note that non-comparative fairness notions have their drawbacks. For example, if the auditor is discriminatory, then $f$ is no longer a fair relation. Therefore, it is necessary to investigate how traditional fairness notions are related to the notion of non-comparative fairness. One important assumption in our analysis is that we assume that both the human auditor and the algorithm employ the same distance metric $d$ in evaluating the gap between input-output relationships. 

\subsection{Relation with Individual Fairness}\label{Sec - Individual Fairness}
Recall that individual fairness adopts the principle of comparative fairness by comparing two different individuals. In the following proposition, we show how the relation $g$ can be evaluated based on the notion of $(\kappa, \delta)$-individual fairness, when $g$ is non-comparatively fair with respect to another individually fair relation $f$.
\begin{prop}
$g$ is $(\kappa, 2 \epsilon + \delta)$-individually fair, if $g$ is     $\epsilon$-noncomparatively fair with respect to $f$, and $f$ is $(\kappa, \delta)$-individually fair.
\label{Prop: (e, k, d)-IF}
\end{prop}
\begin{proof}
Given $(x_1, y_1)$ and $(x_2, y_2)$ such that $\mathcal{D} \left( x_1,  x_2 \right) \leq \kappa$ (the two individuals are $\kappa$-similar), then $f$ is $(\kappa, \delta)$-individually fair if $d \Big( f(x_1), f(x_2) \Big) < \delta$. However, note that if $g$ is $\epsilon$-noncomparatively fair with respect to $f$, then $ d \Big( g(x_1), f(x_1) \Big) < \epsilon$ and $d \Big( g(x_2), f(x_2) \Big) < \epsilon$. Therefore, by applying a chain of triangle inequalities, we obtain
\begin{equation}
\begin{array}{lcl}
d \Big( g(x_1), g(x_2) \Big) & \leq & d \Big( g(x_1), f(x_1) \Big) + d \Big(  f(x_1), f(x_2) \Big) + d \Big( f(x_2), g(x_2) \Big) 
\\[2ex]
& < & 2 \epsilon + \delta.
\end{array}
\label{Eqn: Prop1 - Final eqn}
\end{equation}
\end{proof}

We illustrate this result using the following example from the banking domain. Consider two individuals who are looking to apply for a loan. An individually fair banking system evaluates both the applications via collecting various customer's attributes such as gender, race, address, credit history, collateral, and his/her ability to pay back. At the same time, consider an auditor who makes fairness judgements based on the rule: "If he/she has cleared all the debts and possesses reasonably valued collateral, the loan must be granted". Given that the auditor treats any two similar individuals similarly, the auditor is individually fair. If the banking evaluation system is relatively similar to the auditor's fair relation, from Proposition \ref{Prop: (e, k, d)-IF}, the banking system is also individually fair. Otherwise, if the banking system is not non-comparatively fair with respect to the auditor, then the system itself is also not individually fair.

\begin{prop}
If $f$ is not $(\kappa, \delta)$-individually fair and if $g$ is $\epsilon$-noncomparatively fair with respect to $f$, then $g$ is not $(\kappa, \delta - 2\epsilon)$-individually fair.
\label{Prop: IF - converse}
\end{prop}
\begin{proof}
If $f$ is not individually fair, then for some input pair $(x_1, x_2)$ such that $\mathcal{D}(x_1, x_2) < \kappa$, we have $d\left(f(x_1), f(x_2)\right) > \delta$ for all $\kappa, \delta \in \mathbb{R}$. However, note that if $g$ is $\epsilon$-noncomparatively fair with respect to $f$, then $d\left(g(x_1), f(x_1)\right) < \epsilon$ and $d\left(g(x_2), f(x_2)\right) < \epsilon$. Therefore, by applying a chain of triangle inequalities, we have
\begin{equation}
\begin{array}{lcl}
d\left(f(x_1), f(x_2)\right) & \leq & d\left(g(x_1), f(x_1)\right) + d\left(g(x_2), f(x_2)\right) + d\left(g(x_1), g(x_2)\right)
\end{array}
\end{equation}
Substituting the bounds of $d\left(g(x_2), f(x_2)\right)$ and  $d\left(g(x_1), g(x_2)\right)$ we get
\begin{equation}
\begin{array}{lcl}
2 \epsilon + d\left(g(x_1), g(x_2)\right) & > & d\left(g(x_1), f(x_1)\right) + d\left(g(x_2), f(x_2)\right) + d\left(g(x_1), g(x_2)\right)
\\[2ex] 
& \geq & d\left(f(x_1), f(x_2)\right) > \delta
\end{array}
\end{equation}
for all $\delta \in \mathbb{R}$. Therefore, we also have
\begin{equation}
\begin{array}{lcl}
d\left(g(x_1), g(x_2)\right) & > & \delta - 2\epsilon.
\end{array}
\end{equation}
\end{proof}

Consider the earlier example of banking where, there are two individuals, $A$ and $B$, who possess the same degree of merit. Imagine that the bank approves $A$'s loan application and denies $B$. This outcome remains the same as per the auditor's fair relation. Imagine further that neither $A$ nor $B$ merits the outcome. Though both banking's evaluation and auditor's judgements seem to be similar, they violate the precept, "treat similar individuals similarly". Moreover, the outcome violates the principle of non-comparative justice, since $A$ is treated in a way that $A$ does not merit. Hence, we can assert that banking evaluation does not satisfy individual fairness.

\subsection{Relation with Group Fairness}\label{sec: Group Fairness}
We define a weaker definition for group fairness in the following manner:
\begin{defn}[Coarse Statistical Parity]
Given set of protected attributes $A$, $f$ satisfies $\delta$-statistical parity when 
\begin{equation}
\mathbb{P}[f(x) = y \ | \ A = a] - \mathbb{P}[f(x) = y \ | \ A = a'] \leq \delta
\end{equation}
holds true for all $y \in Y$, $x \in X$, and $a, a' \in A$.
\label{Defn: delta-GF}
\end{defn}

We also define a weaker version of equal opportunity as follows.
\begin{defn}[Coarse Equal Opportunity]
Given a set of protected attributes $A$ and outcome label $y$, $f$ satisfies $\delta$-equal opportunity when
\begin{equation}
\mathbb{P}[f(x) = 1 \ | \ y = 1, A = a] - \mathbb{P}[f(x) = 1 \ | \ y = 1, A = a'] \leq \delta.
\end{equation}
\end{defn}

Analogously, a weaker version of calibration can be defined as follows.
\begin{defn}[Coarse Calibration]
A binary predictor $f$ satisfies calibration given a set of protected attributes $A$ and outcome, $Y = f(x)$, if
\begin{equation}
\mathbb{P}[y = 1 \ | \ f(x) = 1, A = a] - \mathbb{P}[y = 1 \ | \ f(x) = 1, A = a'] \leq \delta.
\end{equation}
\end{defn}

Note that the above definitions equates to Definitions \ref{Defn: GF}, \ref{Defn: EO}, \ref{Defn: Calibration} respectively, when the difference between the probabilities is equal to zero. As discussed earlier, group fairness notions resembles the principle of comparative justice by comparing certain probabilistic measure across two protected groups. In the remaining section, we will focus on the relationship between group and non-comparative fairness notions. For the sake of convenience, let us denote $p_{x,y}(g,a) = \mathbb{P}[g(x) = y \ | \ A = a]$.
\begin{prop}
Given that the probability distributions are $M$-Lipschitz continuous over all possible $f$ and $g$ functions, $g$ satisfies $(2M\epsilon + \delta)$-statistical parity, if $g$ is $\epsilon$-noncomparatively fair with respect to $f$, and $f$ satisfies $\delta$-statistical parity.
\label{Prop: GF}
\end{prop}
\begin{proof}
Given the set of protected attributes $\mathcal{A}$, since $f$ satisfies $\delta$-statistical parity, we have $ || p_{x,y}(f,a) - p_{x,y}(f,a') || < \delta$ for all $a, a' \in \mathcal{A}$. Then, we have
\begin{equation}
\begin{array}{lcl}
p_{x,y}(g, a) - p_{x,y}(g, a') & = & \left[ p_{x,y}(g, a) - p_{x,y}(f, a) \right] + \left[ p_{x,y}(g, a') - p_{x,y}(f, a') \right]
\\[1ex]
&& \quad \quad + \left[ p_{x,y}(f, a) - p_{x,y}(f, a') \right]
\end{array}
\end{equation}
Assuming $M$-Lipschitz continuity over all $f(x)$, $g(x)$, we have $|| p_{x,y}(g, a) - p_{x,y}(f, a) || < M \cdot \epsilon$, since $d (g(x), f(x) ) < \epsilon$. Combining all the inequalities, we have 
\begin{equation}
|| p_{x,y}(g, a) - p_{x,y}(g, a') || < 2 M \epsilon + \delta.
\label{Res: GF}
\end{equation}
\end{proof}
Again, consider the earlier example of loan approvals to illustrate the above proposition. Consider that there exists two groups which are classified based income - low and high. The banking system builds a credit model based purely. Moreover, the system may decide to use different requirement levels - low interest or default to low income group, so that the percentage of people getting a loan in low-income group is equal to the percentage of people getting a loan in high-income group. Now, suppose an auditor presents fair judgements based on the rule: “If Group A has a FICO credit score of 550 and cleared all the debts, the loan must be granted. If Group B has a FICO score of 700 and has valuable collateral, grant the loan”. Note that, the auditor's fair relation is somewhat similar to that of the bank's policy. Since the bank's policy is known to be statistically fair, the auditor is also unbiased from a group fairness perspective.  

Similarly, the following two propositions identify the relationship between our proposed non-comparative notion and two other group fairness notions, namely coarse equal opportunity and coarse calibration.
\begin{prop}
Given that the probability distributions are $M$-Lipschitz continuous over all possible $f$ and $g$ functions, $g$ satisfies $(2M\epsilon + \delta)$-equal opportunity, if $g$ is $\epsilon$-noncomparatively fair with respect to $f$, and $f$ satisfies $\delta$-equal opportunity.
\label{Prop: EO}
\end{prop}
\begin{proof}
The proof is similar to that of Proposition \ref{Prop: GF}. Therefore, for the sake of brevity, the proof is not included. 
\end{proof}

\begin{prop}
Given that the probability distributions are $M$-Lipschitz continuous over all possible $f$ and $g$ functions, $g$ satisfies $(2M\epsilon + \delta)$-calibration, if $g$ is $\epsilon$-noncomparatively fair with respect to $f$, and $f$ satisfies $\delta$-calibration.
\label{Prop: Calibration}
\end{prop}
\begin{proof}
The proof is similar to that of Proposition \ref{Prop: GF}. Therefore, for the sake of brevity, the proof is not included. 
\end{proof}

\section{Identification of Fair Auditors}\label{sec: Unknown Biases}
\begin{figure}[t]
    \centering
    \includegraphics[width=\textwidth]{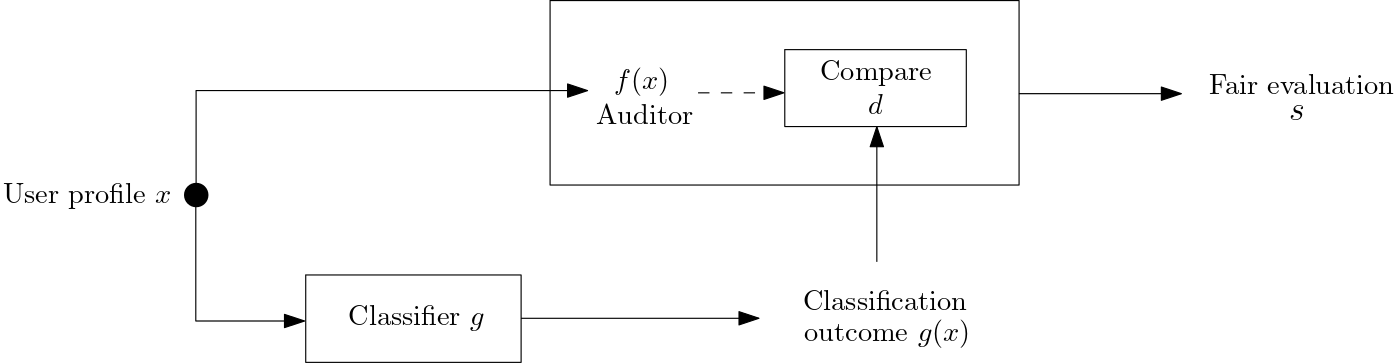}
    \caption{Identifying fair/unfair classifier using a fair auditor}
    \label{fig:NCF Architecture}
\end{figure}
Several attempts have been made to identify how people perceive fairness in order to automate the process of fixing fairness from an algorithmic standpoint \cite{binns2019human,grgic2018human,greenberg1986determinants, saxena2019fairness, srivastava2019mathematical}. However, the underlying assumption in most of the fairness literature is that these systems are evaluated by fair and unbiased auditors. This is not always true because people exhibit a wide range of biases based on diverse prior experiences. As a solution to this problem, we will demonstrate how the proposed non-comparative fairness notion can be used to identify fair auditors. As shown in Figure \ref{fig:NCF Architecture}, we compare an auditor's evaluation to a benchmark entity (e.g. system, or a trained human expert) whose biases are well-quantified in terms of both weak individual fairness and weak group fairness notions as defined in Section \ref{sec: NC Fairness}. The auditor employs a judgement $f$ to evaluate the datapoints and presents a binary score $s$ by comparing it with the benchmark entity's output as shown below.
\begin{equation}
s = \begin{cases}
1, & \text{if } \ d(g(x), f(x)) \geq \epsilon,
\\[1ex]
0, & \text{otherwise.}
\end{cases}
\label{Eqn: Auditor Judgement}
\end{equation}
Note that both $f(x)$ and $\epsilon$ are latent parameters of the auditor's decision model in our framework, and are assumed to be unknown to the system. However, given a benchmark system, we can evaluate the biases of an unknown auditor as a function of biases within the benchmark entity using Propositions \ref{Prop: (e, k, d)-IF} and \ref{Prop: GF}, as stated below.

\begin{cor}[to Proposition \ref{Prop: (e, k, d)-IF}]
If the benchmark entity is $(\kappa, \delta)$-individually fair, then the unknown auditor can be deemed $(\kappa, \delta')$-individually fair when the auditor is $\epsilon$-noncomparitively fair with respect to the benchmark entity such that
$$
\epsilon < \displaystyle \frac{\delta' - \delta}{2}.
$$
\end{cor}
\begin{proof}
Let $g$ denote the unknown auditor and $f$ denote the benchmark entity. From Proposition \ref{Prop: (e, k, d)-IF}, if $f$ is $(\kappa, \delta)$-individually fair and $g$ is $\epsilon$-noncomparatively fair with respect to $f$, we have
\begin{equation}
\begin{array}{lcl}
d\left(g(x_1), g(x_2)\right) & < & 2\epsilon + \delta
\end{array}
\end{equation}

However, our goal is to identify auditors who are $(\kappa, \delta')$-individually fair. In other words, we need $\delta'$ to be at least as large as $2\epsilon + \delta$. In other words,
\begin{equation}
\begin{array}{lcl}
\delta' & > & 2\epsilon + \delta.
\end{array}
\end{equation}

Upon rearranging the terms, we get the bound on $\epsilon$ stated in this corollary.

\end{proof}

Similarly, we can also quantify the tolerable bias in an unknown auditor if the goal is to identify a fair auditor from a weak statistical parity sense. This is discussed in the following corollary.
\begin{cor}[to Proposition \ref{Prop: GF}]
Assuming $M$-Lipschitz continuity in probability distributions at both the benchmark entity and the unknown auditor, if the benchmark entity satisfies $\delta$-statistical parity, then the unknown auditor can be deemed to satisfy $\delta'$-statistical parity when the auditor is $\epsilon$-noncomparitively fair with respect to the benchmark entity such that
$$
\epsilon < \displaystyle \frac{\delta' - \delta}{2M}.
$$
\end{cor}
\begin{proof}
Let $g$ denote the unknown auditor and $f$ denote the benchmark entity. From Proposition \ref{Prop: GF}, if all probability distributions are $M$-Lipschitz continuous, $f$ satisfies $\delta$-statistical parity and $g$ is $\epsilon$-noncomparatively fair with respect to $f$, we have
\begin{equation}
\begin{array}{lcl}
||p_{x, y}(g, a) - p_{x, y}(g, a')|| < 2M\epsilon + \delta
\end{array}
\end{equation}

However, our goal is to identify auditors who satisfy $\delta'$-statistical parity. In other words, we need
\begin{equation}
\begin{array}{lcl}
\delta' > 2M\epsilon + \delta.
\end{array}
\end{equation}

Upon rearranging the terms, we obtain the bound on $\epsilon$ as stated in the corollary statement.
\end{proof}
  
\section{Empirical Analysis and Validation \label{sec: Validation}}
In this section, we validate our theoretical findings using real-world data. We experiment with three datasets:

\begin{enumerate}[label=(\roman*)]
\item \textbf{\emph{ProPublica's COMPAS dataset:}} The goal of COMPAS is to predict the defendant's likelihood to re-offend. The output feature is binary (least likely or most likely) and the input features are as follows: age, race, sex, decile score (0 to 10), degree of offence (felony or misdemeanor), and priors count (number of earlier offences). This dataset consists of 7214 data tuples. We perform same preprocessing as the original analysis of ProPublica. The races in the dataset are only restricted to African-American, Caucasian, and other. WE consider Females and Caucasians as privileged groups and 0 (least likely) is considered as favourable outcome. Since the feature \emph{age} is continuous, we create different age groups (e.g. 25-45 or >45) and rename the features as \emph{age category}. Similar grouping is also performed with the feature \emph{priors count} as well. To encode the categorical features (\emph{age category, priors count}, and \emph{charge degree}), we generated dummies for each feature and converted categorical columns to columns of 0s and 1s. Upon preprocessing, the dataset consists of 5278 data tuples. Note that we also consider \emph{decile score} as an output feature to test our approach on $M$-ary classifiers (binary classifier if $M = 2$, and non-binary classifier if $M > 2$).

\item \textbf{\emph{German credit data:}} The task-at-hand is to predict the credit risk (low or high) of an individual. We consider credit history, savings, employment, personal status, and age. Using the feature \emph{personal status}, we create a new column labelled as \emph{sex}. Moreover, the feature \emph{age} is categorized into two groups: young (< 26) and old (>=26). In this dataset, Males and older individuals are considered as privileged members and 1 (good credit risk) is viewed as a favourable outcome. We employed the same dummy variables approach to encode the categorical features. The dataset consists of 1000 data tuples.

\item \textbf{\emph{Adult Census Income dataset (from UCI Data repository):}} The objective is to predict whether the income of an individual is  >\$50K or <\$50K. The input features include age, sex, race, and education. In preprocessing phase, the continuous feature \emph{age} is transformed to different groups of ages (0-10, 11-20, and so on). Regarding the feature \emph{race}, we limited the labels to binary by mapping White to 1 and all other races to 0. We have 32561 data tuples after preprocessing.
\end{enumerate}

Though our framework specifies that the auditor reveals a binary judgement $s \in \{0, 1\}$, for the sake of practical evaluation, we assume that he/she reveals the exact classification of the input in the same space as the respective system. The rest of the section is presented as follows. Firstly, we evaluate whether a given auditor is fair with respect to comparative fairness notions. Once the auditor is identified as fair, he/she can be leveraged as a benchmark entity to evaluate the datasets. 

\subsection{Evaluating Individual Fairness}
Note that, individual fairness notions rely on a similarity metric $\mathcal{D}$ between two individuals. Since the attributes in real-world datasets are correlated to one another, we consider Mahalanobis distance to compute the similarity between two randomly picked individuals, because it measures distances between points considering how the rest of the datapoints are distributed. The squared Mahalanobis distance can be defined as follows.
\begin{equation}
\mathcal{D}^2 = (\boldsymbol{x}_i - \boldsymbol{x}_j)^T C^{-1} (\boldsymbol{x}_i - \boldsymbol{x}_j)
\end{equation}
where $\boldsymbol{x}_i, \boldsymbol{x}_j$ are observations/rows in a dataset and $C$ is positive semi-definite covariance matrix. Initially, we compute the covariance matrix which summarizes the variance of the dataset. The term $(\boldsymbol{x}_i - \boldsymbol{x}_j)$ represents the vector difference between $\boldsymbol{x}_i$ and $\boldsymbol{x}_j$. The maximum Mahalanobis distance between any two individuals in COMPAS dataset is found to be 9.2. On the other hand, to compute the distance between the outcomes, $d(\cdot)$, we consider absolute difference as the distance metric. 

From the perspective of binary classifiers, a system/entity can only comply with non-comparative fairness when it's evaluations are exactly similar to the auditor's judgements. In other words, the distance between the system's evaluation and auditor's judgement, $d\left(g(x), f(x)\right)$, is upper bounded by 1 $(=\epsilon)$. Therefore, from the Equation \eqref{Eqn: Prop1 - Final eqn} we have, $d \Big( g(x_1), g(x_2) \Big) \leq 2 + \delta$. Moreover, since the distance between the outcomes, $d$, is considered as the absolute difference, $d \in \{0, 1\}$ always. Hence, Equation \eqref{Eqn: Prop1 - Final eqn} will always be true for any given pair of individuals. To avoid this, we only consider m-ary classifiers to evaluate individual fairness. We evaluate the COMPAS dataset with \emph{decile score} (on a scale of 1 to 10) as the output feature. 

\begin{figure}
    \centering
    \begin{minipage}{0.5\textwidth}
        \centering
        \includegraphics[width=\textwidth]{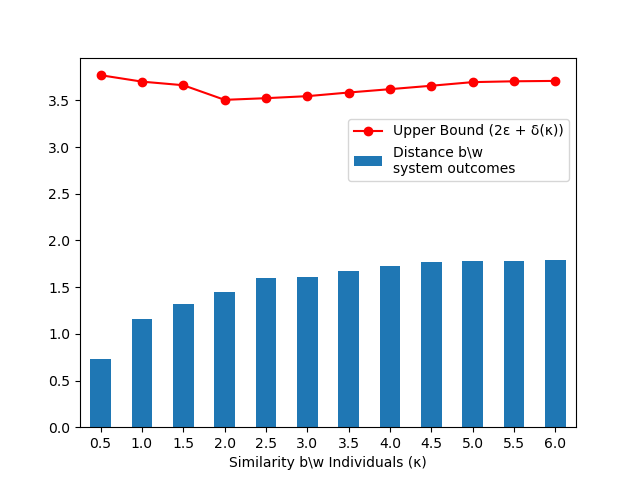}
    \end{minipage}\hfill
    \begin{minipage}{0.5\textwidth}
        \centering
        \includegraphics[width=\textwidth]{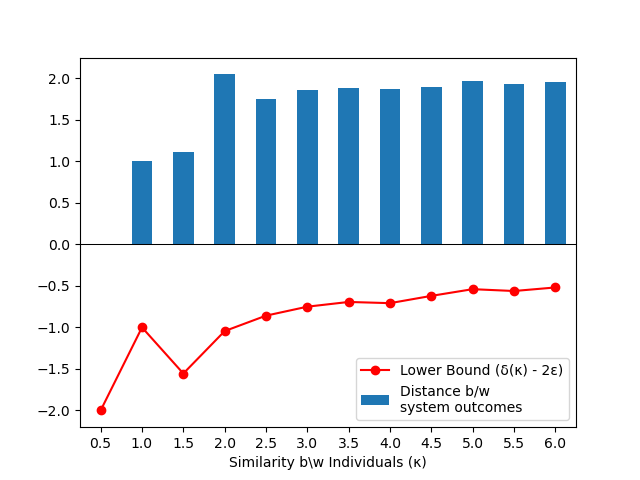}
    \end{minipage}
    \caption{Evaluating COMPAS dataset using Individual Fairness with respect to the Auditor}
\label{Fig: IF-COMPAS}
\end{figure}

Consider the following illustrative example for auditor's judgement to evaluate COMPAS dataset.
\begin{equation*}
f_{COMPAS}(x) = 
\begin{cases}
1, & \text{if $x$.priors-count = 0 and $x$.charge-degree = M}
\\[1ex]
2, & \text{if $x$.priors-count = 0 and $x$.charge-degree = F} 
\\[1ex]
3, & \text{if $x$.priors-count $\in [1, 3]$ and $x$.charge-degree = F}
\\[1ex]
4, & \text{if $x$.priors-count $\in [4, 7]$, $x$.age-category $> 45$ and $x$.charge-degree = M} 
\\[1ex]
5, & \text{if $x$.priors-count $\in [4, 7]$, $x$.age-category $> 45$ and $x$.charge-degree = F} 
\\[1ex]
6, & \text{if $x$.priors-count $\in [4, 7]$ and $x$.age-category $\in [16-45]$} 
\\[1ex]
7, & \text{if $x$.priors-count $\in [8, 15]$ and $x$.age-category $> 45$}
\\[1ex]
8, & \text{if $x$.priors-count $\in [8, 15]$ and $x$.age-category $\in [16-45]$}
\\[1ex]
9, & \text{if $x$.priors-count $\in [16, 24]$ and $x$.age-category $> 45$}
\\[1ex]
10, & \text{if $x$.priors-count $> 24$}
\end{cases}
\end{equation*}

Note that the above relation of the auditor does not satisfy individual fairness as stated in Proposition \ref{Prop: IF - converse}, since there is no $(\kappa, \delta)-$pair that satisfies at every input pair (since $\max(\kappa) = 9.2$ and $\max(\delta) = 9$), as shown in Fig. \ref{Fig: IF-COMPAS}. 

\subsection{Evaluating Group Fairness Notions}

We consider German credit, Adult income, and COMPAS datasets to evaluate group fairness notions. Note that, for COMPAS dataset, the binary feature \emph{two-year-recid} (most likely or least likely) is viewed as the output feature. According the Propositions \ref{Prop: GF}, \ref{Prop: EO}, and \ref{Prop: Calibration}, the auditor must satisfy $\delta$-group fairness. We construct different fair relations for the auditor for different datasets as follows. Firstly, for COMPAS dataset, we only consider \emph{priors count} (number of prior offences) and \emph{charge degree} (degree of the offence) to construct the fair relation.
\begin{equation*}
f_{COMPAS}(x) = 
\begin{cases}
1, & \text{if $x$.priors-count $\in [1,3]$ and $x$.charge-degree = F }
\\
& \qquad \qquad \qquad \qquad \text{OR}
\\
& \text{if $x$.priors-count $> 3$ and $x$.charge-degree = M}
\\[1ex] 
0, & \text{otherwise.} 
\end{cases}
\end{equation*}
Since the task of the Adult income dataset is to predict whether yearly income of an individual is >\$50K or <=\$50K, we consider the feature \emph{education} in the auditor's fair relation as shown below.
\begin{equation*}
f_{Adult}(x) = 
\begin{cases}
1, & \text{if $x$.education $\in$ [Bachelors, Masters, School Professor, Doctorate]}
\\[1ex] 
0, & \text{otherwise.} 
\end{cases}
\end{equation*}
Similarly, for German credit dataset, the features \emph{savings, credit history} and \emph{employment} are considered while designing the auditor's relation.
\begin{equation*}
f_{Credit}(x) = 
\begin{cases}
1, & \text{if $x$.savings $> 500$, $x$.credit-history = Paid, and $x$.employment $> 2$ years}
\\[1ex] 
2, & \text{otherwise.} 
\end{cases}
\end{equation*}

Having defined auditor's evaluation functions, we now demonstrate whether the auditor is fair with respect to different group fairness across every dataset. More specifically, the auditor's evaluations must satisfy the weaker definitions of group fairness notions defined earlier. Statistical parity difference is computed as the difference of the rate of favorable outcomes received by the unprivileged group to the privileged group. The ideal value of this metric is 0. The statistical parity difference of auditor's evaluations across different protected features is presented in Table \ref{Table: Auditor GF Eval}. We observe that the statistical parity differences across different datasets based on respective auditor's fair relation are significantly close to 0. Therefore, we can accurately say that the auditor satisfies statistical parity with respect to every dataset. On the other hand, equalized odds difference is computed as the difference of true positive rates between the unprivileged and the privileged groups. The ideal value is 0. However, a value of $< 0$ implies higher benefit for the privileged group and a value $> 0$ implies higher benefit for the unprivileged group. Table \ref{Table: Auditor GF Eval} shows that the auditor performs poorly with COMPAS dataset when \emph{sex} is considered as the protected feature. Whereas, the probabilistic differences with Adult income and German credit datasets are actually close to 0. Hence, we can assert that the auditor satisfies equal opportunity regarding all the datasets. Lastly, calibration difference is calculated as the difference of positive predictive value between the unprivileged and the privileged groups. Observations from Table \ref{Table: Auditor GF Eval} reveals that the auditor's evaluations satisfy calibration across all the datasets.

\begin{table}[t]
\caption{Evaluating auditor's fair relations with respect to group fairness notions}
\centering
\begin{adjustbox}{width=\textwidth}
\begin{tabular}{ |m{4cm}|m{1cm}|m{1cm}|m{1cm}|m{1cm}|m{1cm}|m{1cm}| }
\hline
\\[-1em]
& \multicolumn{2}{|c|}{COMPAS} & \multicolumn{2}{|c|}{Adult Income} & \multicolumn{2}{|c|}{German Credit}
\\\hline
\\[-1em]
& sex & race & sex & race & sex & age
\\\hline
\\[-1em]
Statistical Parity Difference & -0.05 & -0.02 & 0.02 & -0.05 & 0.03 & 0.07
\\\hline
\\[-1em]
Equal Opportunity Difference & 0.12 & 0.08 & 0.01 & 0.01 & 0.01 & 0.01
\\\hline
\\[-1em]
Calibration Difference & -0.003 & 0.09 & 0.01 & 0.01 & 0.01 & 0.01
\\\hline
\end{tabular}
\end{adjustbox}
\label{Table: Auditor GF Eval}
\end{table}

Leveraging the fair auditor as a benchmark entity, we now evaluate the datasets using the proposed notion of $\epsilon$-noncomparative fairness. Note that, in case of binary classifiers, the distance between system's evaluation and auditor's judgement must be equal to 0. In other words, the distance, $d(g(x), f(x))$, is upper bounded by 1 $(= \epsilon)$. We now demonstrate whether the datasets comply with group fairness notions with respect to the auditor. Propositions \ref{Prop: GF}, \ref{Prop: EO}, and \ref{Prop: Calibration} specify that the probability distributions are $M$-Lipschitz continuous over all possible $f$ and $g$ functions. Since $\epsilon=1$, the real constant $M$ is lower bounded by the probability difference, $ || p_{x,y}(g,a) - p_{x,y}(f,a) ||$, over all protected features. Moreover, since the probability differences are computed between underprivileged to privileged groups, we consider the exact difference rather than the absolute value. We illustrate the comparison of system's evaluations and auditor's judgements in Figure \ref{Fig: comparison}. Table \ref{Table: Sys GF Eval} summarizes validation results of every dataset across different notions. The outcome distance and the upper bound in Table \ref{Table: Sys GF Eval} indicate the left-hand and right-hand terms in the Equation \eqref{Res: GF} respectively. For COMPAS dataset, it is evident that the system's evaluations are significantly different from auditor's judgements with respect to statistical parity. However, we have seen that the auditor does comply with statistical parity. In other words, COMPAS does not satisfy statistical parity satisfy, because it does not comply with $\epsilon$-noncomparative fairness. As a result, the outcome distance is greater than the upper bound (Table \ref{Table: Sys GF Eval} column \emph{sex} under COMPAS). Unfortunately, similar results can be observed in equal opportunity and calibration notions across different protected attributes with respect to the auditor. Analogously, we validate Adult income and German credit datasets. Interestingly, none of the datasets satisfy statistical parity across both the protected attributes. On the other hand, Adult income and German credit datasets does comply with calibration if \emph{sex} is considered as the protected attribute.

\begin{table}[t]
\caption{Validating different datasets for group fairness notions with respect to the auditor}
\centering
\begin{adjustbox}{width=\textwidth}
\begin{tabular}{ |m{3.5cm}|m{1cm}|m{1cm}|m{1cm}|m{1cm}|m{1cm}|m{1cm}|m{1cm}|m{1cm}|m{1cm}|m{1cm}|m{1cm}|m{1cm}|}
\hline
\\[-1em]
& \multicolumn{4}{|c|}{COMPAS} & \multicolumn{4}{|c|}{Adult Income} & \multicolumn{4}{|c|}{German Credit}
\\\hline
\\[-1em]
& \multicolumn{2}{|c|}{sex} & \multicolumn{2}{|c|}{race} & \multicolumn{2}{|c|}{sex} & \multicolumn{2}{|c|}{race} & \multicolumn{2}{|c|}{sex} & \multicolumn{2}{|c|}{age}
\\\hline
\\[-1em]
& Outcome Distance & Upper Bound & Outcome Distance & Upper Bound & Outcome Distance & Upper Bound& Outcome Distance & Upper Bound& Outcome Distance & Upper Bound& Outcome Distance & Upper Bound
\\\hline
\\[-1em]
Statistical Parity Difference & -0.13 & -0.29 & -0.13 & -0.29 & -0.19 & -0.66 & -0.10 & -0.7 & -0.07 & -0.32 & -0.14 & -0.38
\\\hline
\\[-1em]
Equal Opportunity Difference & -0.26 & -0.48 & -0.24 & -0.36 & 0.46 & -1.5 & 0.18 & -1.36 & 0.002 & 1.7 & -0.38 & -0.74 
\\\hline
\\[-1em]
Calibration Difference & -0.08 & -0.16 & -0.07 & -0.1 & 0.001 & 0.001 & 0.03 & -0.71 & 0.001 & 1.34 & -0.12 & -0.58
\\\hline
\end{tabular}
\end{adjustbox}
\label{Table: Sys GF Eval}
\end{table}

\begin{figure}[t]
\centering
\includegraphics[width=.32\textwidth, trim={0.75cm 0 1cm 0.75cm},clip]{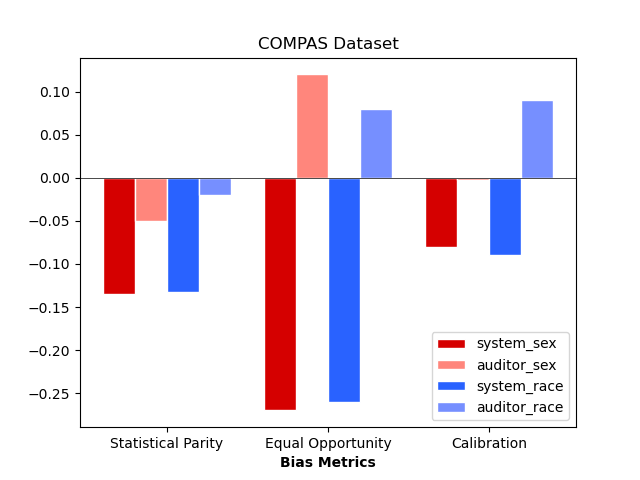}
\includegraphics[width=.32\textwidth, trim={1cm 0 1cm 0.75cm},clip]{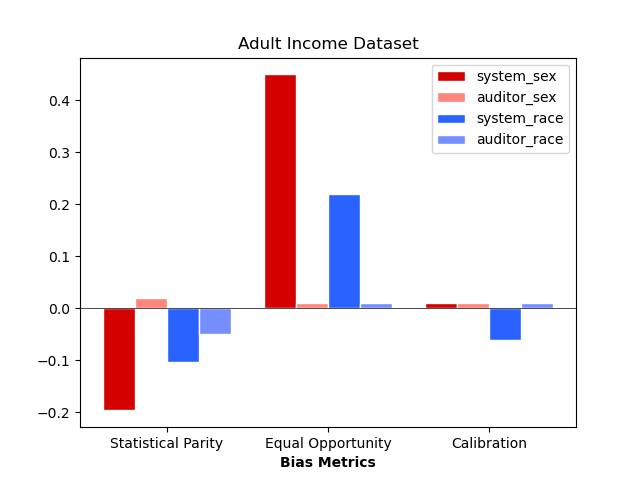}
\includegraphics[width=.32\textwidth, trim={1cm 0 1cm 0.75cm},clip]{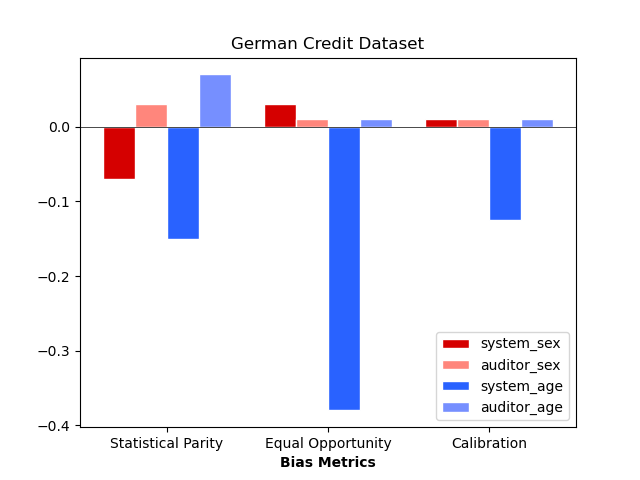}
\caption{Comparing System's Evaluation with Auditor's Judgements}
\label{Fig: comparison}
\end{figure}

\section{Conclusion and Future Work}
We introduced a non-comparative fairness notion which complements the existing comparative fairness notions proposed in the algorithmic fairness literature. We showed that any system can be deemed fair from the perspective of comparative fairness (e.g. individual fairness and statistical parity) if it is non-comparatively fair with respect to an auditor who has been deemed fair with respect to the same fairness notions. We also proved that the converse holds true in the context of individual fairness. We discussed how the proposed non-comparative fairness notion can be used to identify fair auditors who are hired to evaluate latent biases in decision-support systems. We also presented corroborating validation results on three real datasets. In our future work, we will develop novel algorithms to identify fair auditors from the perspective of multiple (potentially intransitive) attributes, and also validate our theoretical findings using real data.

\bibliographystyle{IEEEtran}
\bibliography{sample-bibliography}

\end{document}